\documentclass[twoside]{article}
\usepackage[accepted]{aistats2023}
\usepackage[utf8]{inputenc}
\usepackage{amsmath,amsfonts,amssymb,amsthm}
\usepackage{mathtools} 
\usepackage[english]{babel}
\usepackage{times}
\usepackage[labelfont=sc]{caption}
\usepackage{dsfont}

\usepackage[round]{natbib}
\usepackage{hyperref}
\hypersetup{
    colorlinks=true,
    citecolor=black,
    linkcolor=black,
    filecolor=black,      
    urlcolor=black,
}
\let\originalleft\left
\let\originalright\right
\renewcommand{\left}{\mathopen{}\mathclose\bgroup\originalleft}
\renewcommand{\right}{\aftergroup\egroup\originalright}

\theoremstyle{plain}

\newtheorem*{result}{Result}
\newtheorem*{ass}{Assumption}
\newtheorem{lemma}{Lemma}[section]

\theoremstyle{definition}
\newtheorem{defn}{Definition}[section]

\theoremstyle{remark}
\newtheorem{remark}{Remark}[section]
\newtheorem*{remark*}{Remark}

\newcommand{\E}{\mathbb{E}}

\newcommand{\R}{\mathbb{R}}

\renewcommand{\vec}[1]{\boldsymbol{#1}}
\newcommand{\hvec}[1]{\hat{\vec{#1}}}
\newcommand{\1}{\mathds{1}}

\DeclareMathOperator*{\argmax}{argmax}

\DeclareMathOperator{\sgn}{sgn}

\DeclareMathOperator{\diag}{diag}
\DeclareMathOperator{\tr}{Tr}

\newcommand{\iid}{\stackrel{i.i.d}{\sim}}

\newcommand{\ra}{\rightarrow}

\newcommand{\round}[1]{\left( #1 \right)}
\newcommand{\bround}[1]{\Big( #1 \Big)}

\newcommand{\br}[1]{\left\{ #1 \right\}}
\newcommand{\bbr}[1]{\Big\{ #1 \Big\}}
\newcommand{\bro}[1]{\{ #1 \}}
\newcommand{\norm}[1]{\left \Vert #1 \right \Vert}
\newcommand{\normo}[1]{\Vert #1 \Vert}

\newcommand{\innero}[1]{\langle #1 \rangle}

\newcommand{\ddd}{,\dots,}

\newcommand{\EE}[1]{\E\left[ #1 \right]}

\newcommand{\pmat}[1]{\begin{pmatrix} #1 \end{pmatrix}}

\newcommand{\nn}{\nonumber \\}

\begin{document}
\runningtitle{Semi-supervised MTL on Gaussian mixture}

\twocolumn[

\aistatstitle{Asymptotic Bayes risk of semi-supervised multitask learning\\ on Gaussian mixture}

\aistatsauthor{ Minh-Toan Nguyen \And Romain Couillet }

\aistatsaddress{ GIPSA-lab, Université Grenoble Alpes \And  LIG-lab, Université Grenoble Alpes} ]

\begin{abstract}
  The article considers semi-supervised multitask learning on a Gaussian mixture model (GMM). Using methods from statistical physics, we compute the asymptotic Bayes risk of each task in the regime of large datasets in high dimension, from which we analyze the role of task similarity in learning and evaluate the performance gain when tasks are learned together rather than separately. In the supervised case, we derive a simple algorithm that attains the Bayes optimal performance.
\end{abstract}
\section{INTRODUCTION}
Multitask learning (MTL) is a machine learning method in which multiple tasks are learned simultaneously. It can facilitate knowledge transfer between tasks and can lead to more informative data representation \citep{ruder2017overview}. Although learning from related tasks can help disseminate useful information learned from one task to other tasks, the presence of unrelated tasks can also be beneficial. With the prior knowledge that two given tasks are unrelated, the algorithm can learn to ignore irrelevant features of the data distribution, resulting in better data representation \citep{paredes2012exploiting}.

In this work, we propose a simple model of MTL based on Gaussian mixtures that focuses on capturing the transfer of knowledge between tasks, leaving out the data representation aspect. Our paper extends the semi-supervised learning model studied in \cite{lelarge2019asymptotic}, which examines the added value of unlabeled data in a one-task classification. We consider here instead multiple classification tasks, for which the data in each task are partially labeled and come from two classes. Thanks to the simplicity of our model, we can define the correlation between two tasks as a number in $ [-1,1]$. We are interested in the performance gain when correlated tasks are learned together versus when they are learned separately, assuming the best algorithm is used. This leads to the concept of \emph{Bayes risk}, defined as the minimal feasible probability of misclassifying a new data point not from the training dataset. Despite the randomness of data, in the limit where both the quantity and the dimensionality of the data are large with a fixed ratio, the Bayes risk converges towards a deterministic value.

Although the main objective of this study is to compute the minimum classification error, it is important to emphasize that the posterior distribution of a signal given the observed data is a more fundamental object, as it serves as a basis for deriving optimal estimators with respect to certain criteria. In the high-dimensional regime, the posterior law of a signal is a high-dimensional integral, and despite its complexity, it behaves like a simpler law. This property enables the exact calculations obtained in this work.

\textbf{Contributions and related works.}

As a first contribution, we derive an exact formula for the asymptotic Bayesian risk, based on a simple argument that is similar to the cavity method from statistical physics \citep{mezard2009information}. Although not fully rigorous, the paper aims to provide a clear intuition of the asymptotic equivalence that occurs in high dimensions. This concept underlies most of the equations presented in the paper. The paper is designed to be accessible, and no prior knowledge of physics is required to understand its contents. Our work aligns with a body of research that studies the fundamental limit of various high-dimensional statistical models, including tensor models \citep{barbier2017layered, lesieur2017statistical, lelarge2019fundamental}, generalized linear model \citep{barbier2019optimal} and Gaussian mixture model \citep{lesieur2016phase,lelarge2019asymptotic}. 

Secondly, we analyze the role of task correlations and how they interact with other elements of the model, such as the proportion of labeled data in each task. It is well known that unsupervised learning on a single task with Gaussian mixture data leads to a phase transition that separates the high and low noise regimes. We demonstrate that phase transition persists to the case of multitask and study how it is affected by task correlations. In the context of source task - target task, we identify the conditions in which the source task is most beneficial to the target task.

Finally, we derive a simple algorithm that achieves the optimal performance in the case of supervised learning. Although an optimal performance on a synthetic data set does not necessarily have a good performance on real data, this algorithm shows how the optimal algorithms on separate tasks should be modified when correlations are taken into account. This could offer useful insights for designing MTL methods in practice.

Although our focus is different, there is some connection between our work and theoretical studies that investigate optimization-based inference on simple data models. These studies compute the exact asymptotic performance of algorithms, and examine how this performance is influenced by factors such as choice of loss function, regularization, and number of model parameters. On the other hand, our work focuses on investigating the fundamental limit of statistical problems regardless of any specific algorithm. It is interesting, however, that in some cases, the optimization-based methods can nearly reach or achieve the optimal performance \citep{mai2021consistent,thrampoulidis2020theoretical,mignacco2020role, loureiro2021learning, aubin2020generalization}. For multitask learning on Gaussian mixtures, \cite{tiomoko2021deciphering} obtain exact asymptotic results for least-square support vector
machine using random matrix theory.

There are several reasons to study the GMM. Besides being amenable to theoretical analysis, it is the simplest model that captures the elements of MTL that we are interested in: task correlation and transferring of information. On the application aspect, it is remarked in \cite{lesieur2016phase} that the Bayesian statistics under GMM as a prior can rediscover several key methods in machine learning such as the K-means or spectral clustering algorithms. There exists a close relationship between Bayesian statistics and algorithms. On one hand, Bayesian interpretations can be established for well-known algorithms such as PCA and SVM \citep{tipping1999probabilistic, polson2011data}, which turn out to be standard estimators on fairly simple data models. On the other hand, by starting with a simple data model and devising an optimal algorithm based on specific criteria, one can enhance existing methods or create new ones.
\citep{bishop1998bayesian,krzakala2012statistical}.

\textbf{Notation:} We use the symbol $\innero{\cdot, \cdot}$ to denote the scalar (or inner) product of vectors. If $\vec{X} = (X_{ij})$, then $\vec{X}_{i \cdot} = (X_{ij})_j$ and $\vec{X}_{\cdot j} = (X_{ij})_i$. For $n \in \mathbb{N}$, we use $[n]$ to denote the set $\br{1,2,\dots,n}$. The notation $\vec D_{\vec{x}}$ represents the diagonal matrix with diagonal elements given by the vector $\vec{x}$. If indexed objects such as $\vec X_i$ are given, then $\vec X$ simply means $(\vec X_i)_i$.

The source code for the simulations in this paper is available at: \href{https://github.com/Minh-Toan/Bayes-risk}{https://github.com/Minh-Toan/Bayes-risk}

\section{MODEL}\label{sec:intro}
We consider $T$ classification tasks, where task $ t$ consists of $ N_t$ data points in $ \R^D$. The $ i$-th data point in task $ t$, denoted by $\vec Y_{ti}$, is given by
\begin{align}\label{eq:model-2}
\vec Y_{ti} = V_{ti} \vec U_{t} +\sigma_t \vec Z_{ti}
\end{align}
where $\sigma_t >0$. The random variables $\vec V, \vec U, \vec Z $ are independent, with
\begin{align*}
V_{ti} &\iid \mathcal U(\br{-1,1}), \\
Z_{ti} &\iid \mathcal N(0, I_D),
\end{align*}
and $ \vec U_1 \ddd \vec U_T$ are chosen uniformly randomly on the unit sphere $S^{D-1} = \br{\vec x \in \R^D, \norm{\vec x}=1} $, conditioned on the event
\begin{align*}
\innero{\vec U_t, \vec U_{t'}}=C_{tt'}, t \neq t'.
\end{align*}
The matrix $\vec C= (\innero{\vec U_t, \vec U_{t'}})_{t,t'=1}^T$ is called the \emph{task-correlation matrix}. It follows from the definition that $\vec C$ is a positive definite matrix with diagonal entries all equal to $1$. The tasks are said to be \emph{connected} if for any two tasks $t$ and $t'$, there is a sequence of tasks $t_1 \ddd t_k$ such that $C_{tt_1}, C_{t_1 t_2} \dots $ $C_{t_k t'} \neq 0$. 

In other words, the data in task $ t$ comes from two classes corresponding to two Gaussian distributions centered at $ \pm \vec U_t$ with the same covariance $ \sigma_t^2 I_D$. The positions of the centers are not known and can only be estimated from the data. The class of a data point $ \vec Y_{ti}$ is indicated by $ V_{ti}$, so each data point has probability $ 1/2$ of belonging to each class. A data point is said to be \emph{labeled} if we know which class it belongs to, otherwise it is \emph{unlabeled}. Independently of all other random variables, each data point in task $ t$ is labeled with probability $ \eta_t$. The cases $\eta_t = 1$ and $\eta_t = 0$ correspond to supervised and unsupervised learning. $ C_{tt'}$ measures the correlation between tasks $ t$ and $ t'$. The parameters $ \lambda_t = 1/\sigma_t^2$ are called the \emph{signal to noise ratio} (SNR). As the SNR increases, the two classes separate and classification is easier. We study the model in the setting where the dimension and the amount of data in each task tends to infinity at a fixed rate $\alpha_t = \lim_{D \ra \infty} N_t/D$, called the \emph{sampling ratio}. Note that the model for semi-supervised learning studied in \cite{lelarge2019asymptotic} corresponds to the case $T=1$.

We have access to the dataset $ \vec Y = (\vec Y_{ti})$, the labels as well as model parameters $ (\sigma_t), (\eta_t), (\alpha_t) $ and $ \vec C $. \footnote{ $\vec \sigma$ and $\vec C$ can indeed be estimated with vanishing errors as $ D \ra \infty$, given that a positive fraction of labeled data is available in each task, i.e. $ \eta_t > 0$ for all $ t$ (Appendix \ref{C}).} Our job is to use that available information to classify a new data point $ \vec Y_{\text{new}}$ in any given task $ t$
\begin{align}
    \vec Y_{\text{new}} = V_{\text{new}} \vec U_t + \sigma_t \vec Z_{\text{new}}
\end{align}

We are interested in the minimal classification error, i.e. the Bayes risk
\begin{align}
    \inf_{\hat V} \mathbb P(\hat V \neq V_{\text{new}})
\end{align}
where the infimum is taken over all estimators of $ V_{\text{new}}$.
\section{RESULTS}
Before presenting the results, we need some definitions that will aid in formulating our findings in a clear and concise manner.
\begin{defn}
The inference of $\vec X \in \R^D$ from the data $\vec Y$ satisfies the \emph{replica symmetric} (RS) property with \emph{overlap} $q$ if in the limit $D \ra \infty$, 
\begin{align}
    \innero{\vec X, \vec X^1}, \, \innero{\vec X^1, \vec X^2}, \, \innero{\vec X, \hat{\vec X}}, \, \normo{\hat{\vec X}}^2
\end{align}
all converge to the same limit $q$, where $\vec X^1, \vec X^2$ are sampled independently from the posterior of $\vec X$ given $\vec Y$, and $\hat{\vec X} = \E[\vec X|\vec Y]$, called the \emph{MMSE estimator} of $\vec X$ given $\vec Y$. \footnote{MMSE stands for minimum mean-squared error.} In some contexts, we use $\hvec X$ to refer to a general estimator of $\vec X$, while the MMSE estimator of $\vec X$ is denoted as $\hvec X_{\text{MMSE}}$.
\end{defn}

This property holds for a wide range of inference problems in the setting where the signal is generated from a known distribution. We assume that this property holds true for our model:
\begin{ass}
$\sigma_t^{-1} \vec U_t|\vec Y$ and $N_t^{-1/2} \vec V_t|\vec Y$ satisfies the RS property for all $t \in [T]$ with the overlaps denoted by $q_{ut}$ and $q_{vt}$ respectively.
\end{ass}
The inclusion of the normalizing factor $\sigma_t$ in the definition of $q_{ut}$ is for the purpose of convenience. 

Later in the paper, we will require the following definition in order to prove the results:
\begin{defn}
Consider the following Gaussian channels
\begin{align}
    Y_i = \sqrt{\lambda_i} X_i + Z_i, \quad i = 1 \ddd n
\end{align}
with inputs $X_i$, outputs $Y_i$ and SNRs $\lambda_i$. Let $\hvec X = \E[\vec X|\vec Y]$. The \emph{overlap functions} $F_{\vec X,i}: \R^n \ra \R$ are defined as 
\begin{align}
    F_{\vec X,i}(\vec \lambda) = \E[ \hat X_i X_i ]= \E[\hat X_i^2]
\end{align}
$F_{\vec X,i}$ is also referred to as the \emph{overlap} of the signal $X_i$.
\end{defn}

The main result of the article unfolds as follows.
\begin{result} \label{theo:main}
i) Under the setting of the model, as $ D \ra \infty$, the Bayes risk converges to
\begin{align*}
   1-\Phi \round{\sqrt{q_{ut}}},
\end{align*}
where $\Phi(t)=\frac1{\sqrt{2\pi}}\int_{-\infty}^t e^{-x^2}dx$ 

ii) The overlaps $q_{ut}, q_{vt}$ satisfies the following equations
\begin{subequations}
    \begin{align}
    q_{ut} &= [\vec M-\vec M(\vec I+\vec D\vec M)^{-1}]_{tt} \label{eq:fixed-point-1}\\
    q_{vt} &= \eta_t + (1-\eta_t)F(q_{ut}) \label{eq:fixed-point-2}
    \end{align}
\end{subequations}
with
\begin{align*}
    \vec M &= \br{ C_{tt'}/\sigma_t \sigma_{t'} }_{t,t'=1}^T\\
    \vec D &= \diag \bro{ \alpha_t q_{vt}}_{t=1}^T \\
    F(q) &= \E[\tanh(\sqrt{q}Z + q)], \quad Z \sim \mathcal N(0,1).
\end{align*}

\end{result}
\begin{remark}
When $q_{ut} = 0$, the Bayes risk of task $t$ is equal to $0.5$, which corresponds to the level of classification error of a random guess. In this case, we say that the classification of task $t$ is \emph{impossible}. On the other hand, if $q_{ut}$ is positive, the classification of task $t$ is said to be \emph{feasible}.
\end{remark}

\begin{remark}
The fixed point equations (\ref{eq:fixed-point-1}) and (\ref{eq:fixed-point-2}) may not uniquely determine the overlaps. Specifically, for unsupervised learning with high SNR, two solutions exist: the zero solution is unstable while the non-zero solution is stable, and the stable solution is naturally chosen as overlaps. In other cases, there is only one solution.
\end{remark}

We can perform a sanity check of the result by considering the following special cases: if the similarity between any two tasks is zero, the result implies that MTL has the same asymptotic Bayes risks as learning task separately, which is obvious since the data from different tasks are independent, while if $ \sigma_t = \sigma$ and $ C_{tt'}=1$ for all $ t,t'$, i.e. the data distributions are identical for all tasks, the asymptotic Bayes risks of all tasks are equal to that of a single task with parameters  $ \alpha = \sum_t \alpha_t$ and~$ \alpha \eta = \sum_t \alpha_t \eta_t $ (Appendix \ref{special}).

\section{CONSEQUENCES}
We present in this section some implications of the main result.
\subsection{Supervised learning.} For supervised learning with only one task, the minimal classification error of a new data point $ \vec Y_{\text{new}}$ is achieved by the estimator $ \hat{V}_{\text{new}} = \sgn(\innero{\vec Y_{\text{new}}, \bar{\vec Y}})$, where $ \bar{\vec Y} = N^{-1} \sum_i V_i \vec Y_i$ \citep{lelarge2019asymptotic}. In the multitask case, if $ \vec Y_{\text{new}}$ is a new data point in task $ t$, the following algorithm achieves the optimal performance:
\begin{enumerate}
    \item Compute
    \begin{align*}
        \bar{\vec Y}_t = \frac{1}{N_t} \sum_{i=1}^{N_t} V_{ti} \vec Y_{ti}
    \end{align*}
    \item Compute
    \begin{align*}
        \tilde{\vec Y}_t = \sum_{s=1}^T a_{ts} \bar{\vec Y}_s
    \end{align*}
    where $ \vec A = (a_{ts})_{t,s=1}^T = \vec M\vec D_{\vec \alpha}(\vec I+\vec M\vec D_{\vec \alpha})^{-1}$.
    \item  The asymptotic Bayes risk is achieved by
    \begin{align} \label{eq:optimal}
        \hat{V}_{\text{new}} = \sgn(\innero{ \vec Y, \tilde{\vec Y}_t }).
    \end{align}
\end{enumerate}
We can see that the optimal estimator for multiple tasks modifies the optimal estimators for separated tasks $\bar{\vec Y}_t$ by taking into account the correlations between tasks as well as their levels of difficulty and the relative sizes, measured by $ \vec C, (\sigma_t)$ and $ (\alpha_t)$ respectively. Interestingly, this optimal algorithm coincides with the method proposed in \cite{tiomoko2021pca} using a different approach.
\begin{figure}[!htb]
    \centering
    \includegraphics[width=0.45\textwidth]{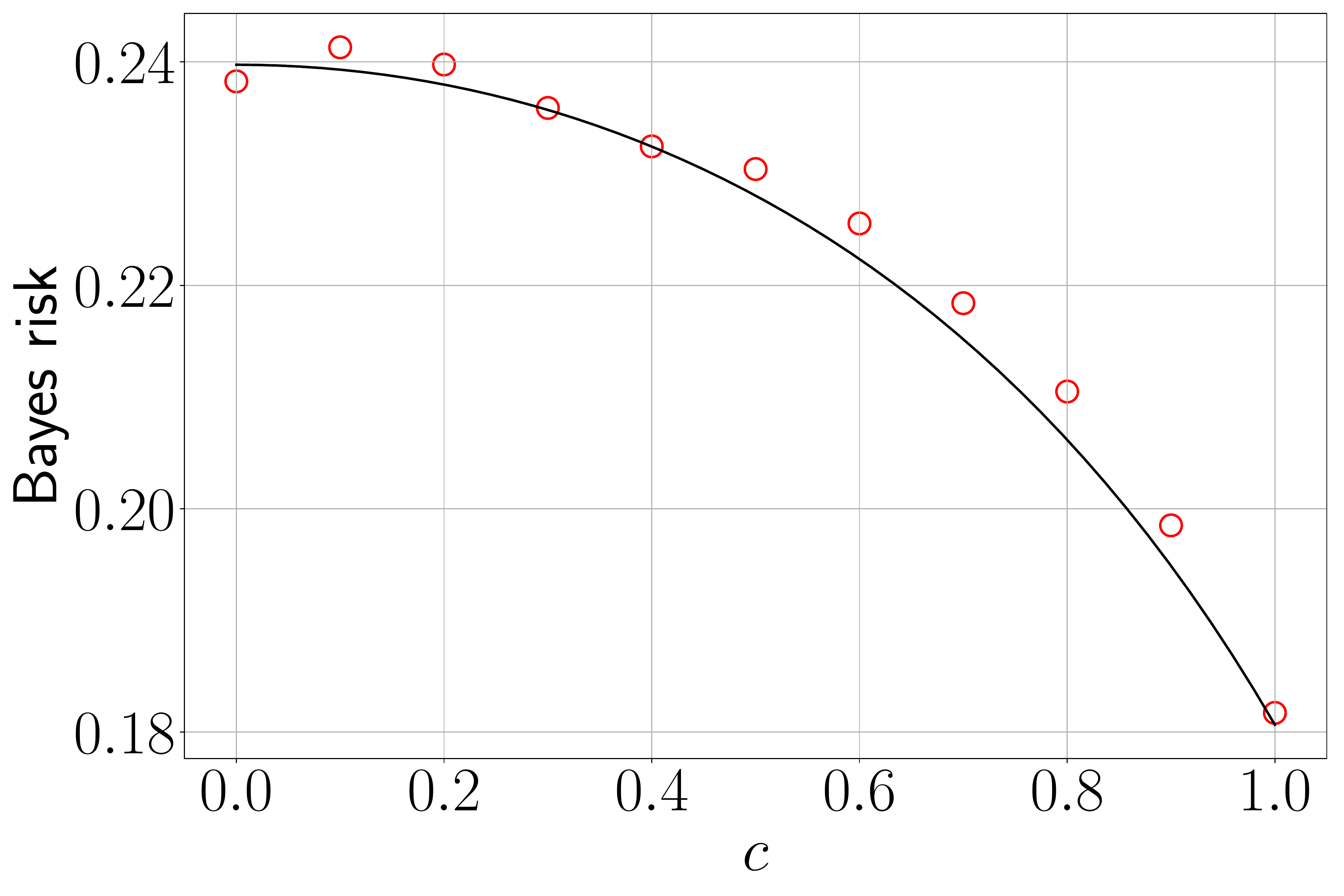}
    \caption{Bayes risk vs performance of the asymptotic optimal algorithm. $\alpha_1 = \alpha_2 = 1$, $\sigma_1 = 1$, $\sigma_2=0.5$, $D=1000$.}
    \label{fig:optimal}
\end{figure}
\medskip

\subsection{Unsupervised learning and phase transition.} A particularly interesting behavior that only occurs in the case of unsupervised learning is phase transition. One of the most well-known example of this phenomenon is \emph{BBP phase transition} \citep{baik2005phase} which concerns a single learning task with $\lim_{D \ra \infty} N/D=1$. When $ \lambda = 1/\sigma^2 \leq 1$, no estimator can achieve a smaller classification error than $ 0.5$. In other words, the classification is objectively impossible since the two classes are statistically identical. On the other hand, we say that a task is \emph{feasible} if one can obtain a classification error smaller than $ 0.5$. It turns out that phase transition persists to the case of multitask. Fig. \ref{fig:phase} shows the performance of task 1 in terms of SNRs in the case of two tasks with $ N_1=N_2=D$ and correlation $ c=0.7$. The classification is impossible in the region delimited by the black curve.
\begin{figure}[!htb]
    \centering
    \includegraphics[width=0.45\textwidth]{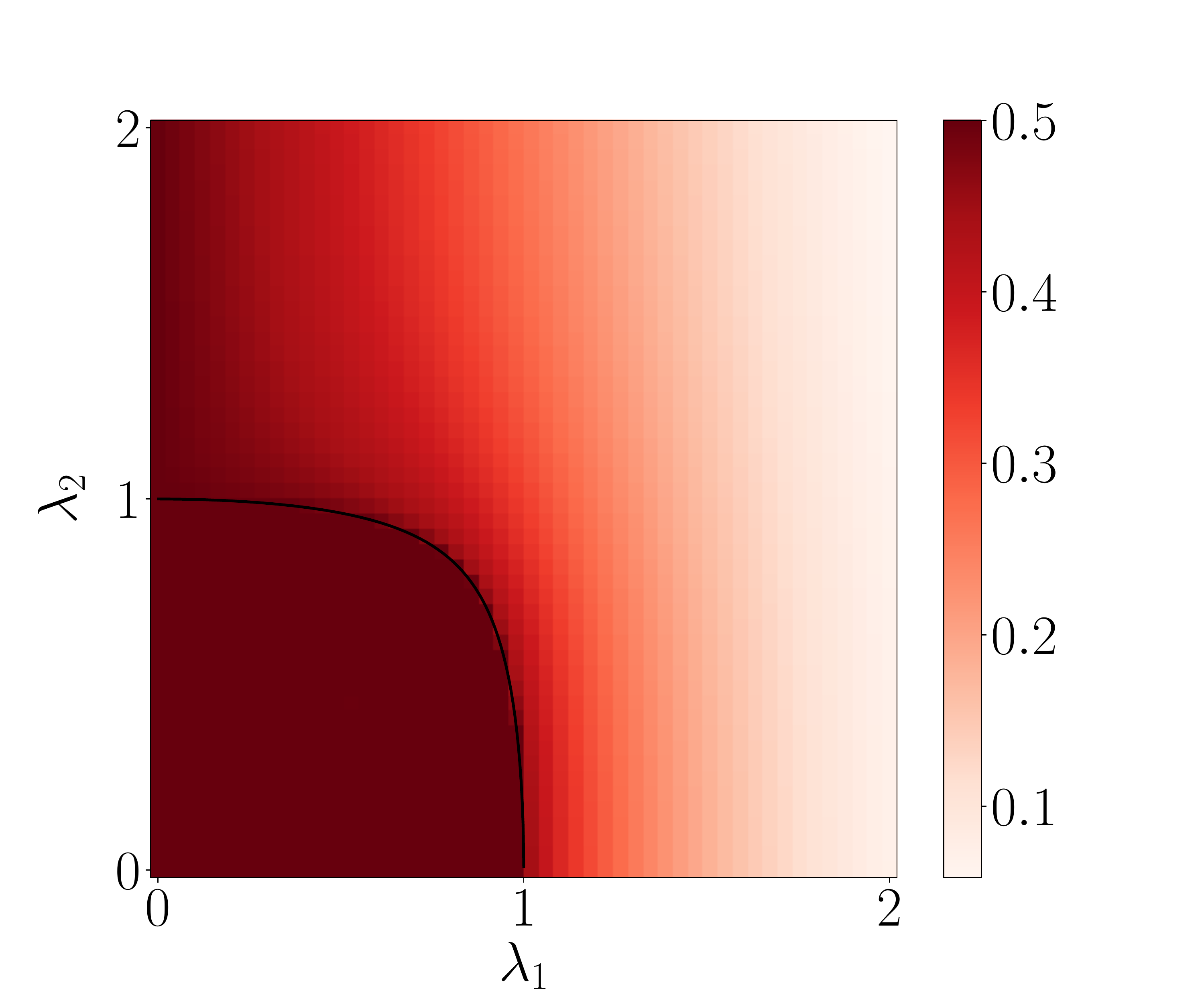}
    \caption{Bayes risk of Task 1 in terms of SNR of each task. Two tasks are unsupervised, with $N_1=N_2=D$ and correlation $ c=0.7$. The classification is impossible in the region delimited by the black curve. The impossible region is identical for two tasks.}
    \label{fig:phase}
\end{figure}

The simulation also shows that the impossible regions are identical for both tasks. In other words, two correlated tasks are either feasible or impossible. In the general case with any number of tasks, tasks are feasible or impossible together, given that they are connected.

Note that phase transition disappears as soon as a positive proportion of labeled data is available, since supervised learning restricted on labeled data already produces a non-trivial performance.

In the case of two tasks with $ N_1=N_2=D$, the region of impossible classification is given by
\begin{align}
    \bbr{ (\lambda_1, \lambda_2) \in [0,1]^2: (1-\lambda_1^2)(1-\lambda_2^2) \geq c^4 \lambda_1^2 \lambda_2^2 }
\end{align}
as shown in Figure~\ref{fig:regions}. As the task correlation $ c$ increases from $ 0$ to $ 1$, this region shrinks from the unit square $ [0,1]^2$ to a quarter of a disk.

\begin{figure}[!htb]
    \centering
    \includegraphics[width=0.45\textwidth]{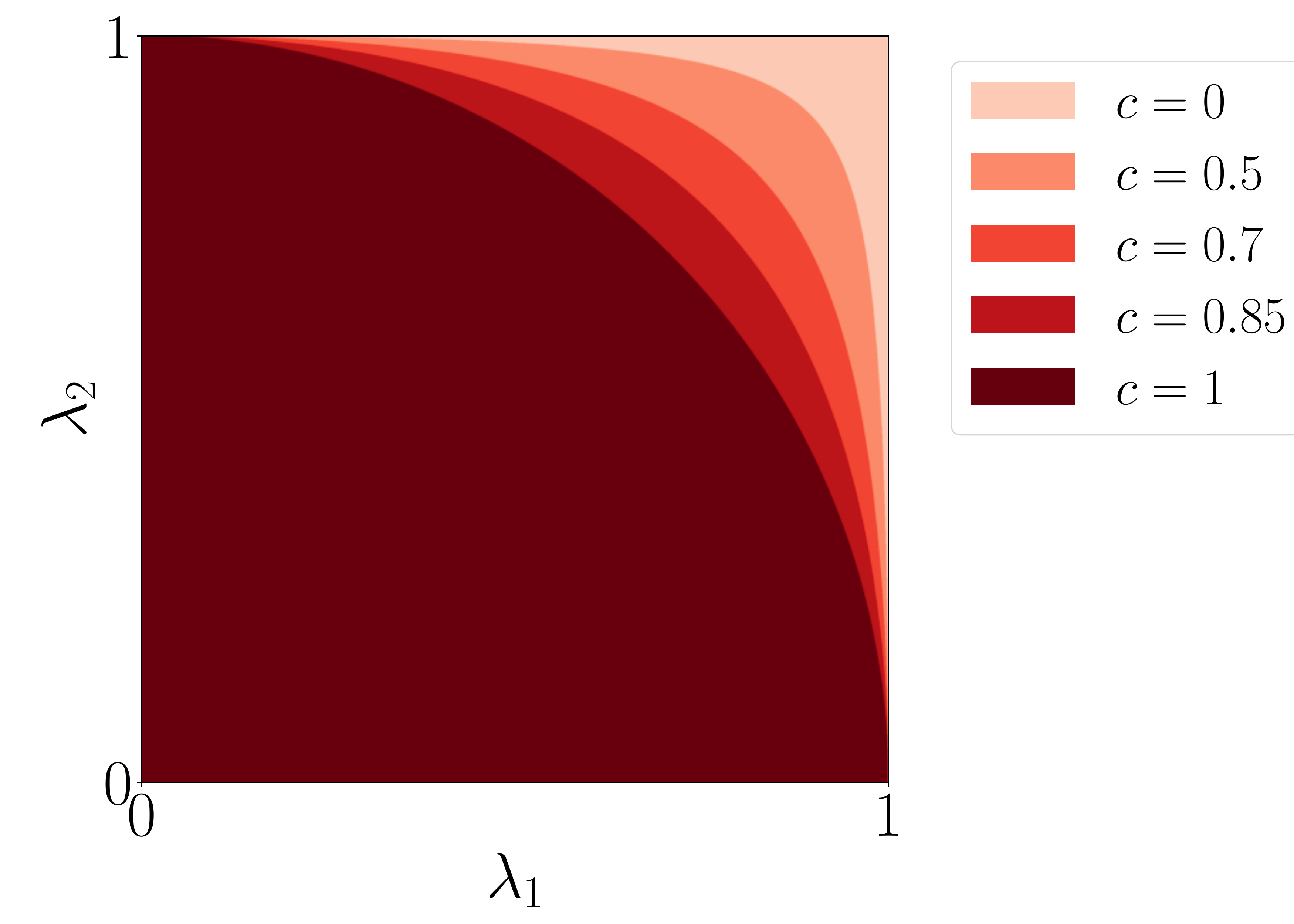}
    \caption{The region of impossible classification shrinks as the task correlation increases. When two tasks are uncorrelated ($c=0$), the region of impossible classification is the whole square $ [0,1]^2$. As $ c$ increases from $0$ to $1$, the impossible region shrinks from the unit square $ [0,1]^2$ to a quarter of a disk.}
    \label{fig:regions}
\end{figure}

Another special case where an explicit formula for the impossible region can be obtained is when there are $ T$ tasks with $ N_1 = \dots = N_T = D$, with correlation $ c > 0$ between any two of them, and $ \lambda_t = \lambda$ for all $ t$. It can be shown that the classification is impossible whenever
\begin{align}
    \lambda \leq \frac{1}{\sqrt{1+(T-1)c^2}}.
\end{align}

\subsection{Semi-supervised learning.} To reduce the number of model parameters in the simulation, we here focus on a specific setting consisting of one \emph{source task} and one \emph{target task}. The source task is comparatively easy: it can be fully labeled, have a high SNR, or have a larger dataset. We want to see how the target task benefits from the source task. 

Figure~\ref{fig:2-tasks} illustrates the effect of task correlation. The task correlation $ c$ ranges from $ 0$ to $ 1$. Note that the correlations $c$ and $ -c$ are essentially the same, since one can be transformed to another by switching labels in one task. The first task (target task) is composed of a small dataset ($ \alpha_1 = 0.1$) without label ($ \eta_1 = 0$), while the second task (source task) consists of a fully labeled dataset ($ \eta_2 = 1$) with twice as much data ($ \alpha_2 = 0.2$). If two tasks are highly correlated ($ c \gtrsim 0.5$), the performance of the target task can be significantly improved. When $ c$ is near zero, the decrease in Bayes risk is slow, in order of $ O(c^2)$. Note that two tasks have the same SNR ($\lambda_1=\lambda_2=4$), so when $c =1$ they have the same data distribution and can be combined into a single task, yielding a identical performance.
\begin{figure}[!htb]
    \centering
    \includegraphics[width=0.45\textwidth]{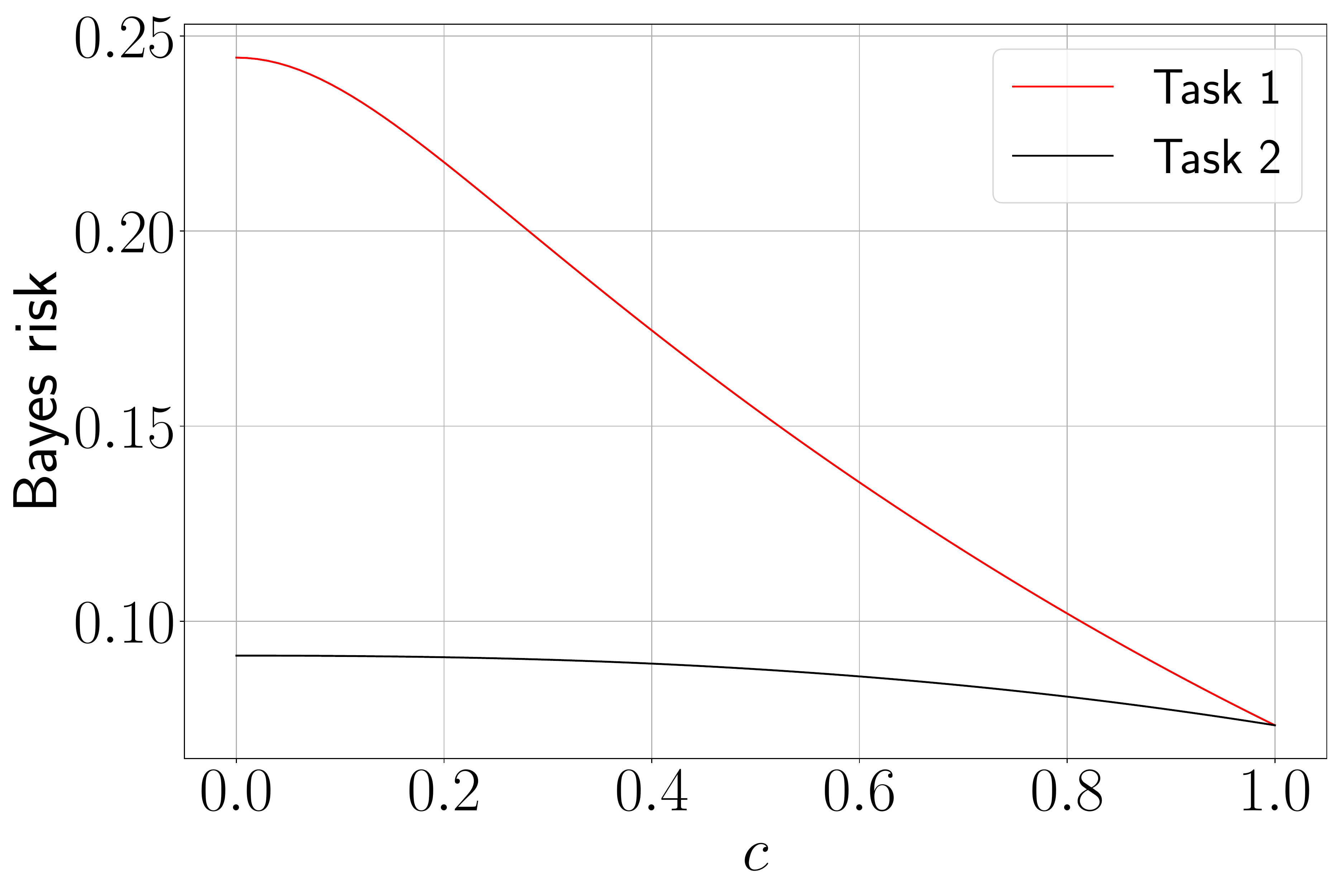}
    \caption{Two-task setting: Bayes risks as a function of the task correlation $c$, with proportions of labeled data $ \eta_1 = 0$, $\eta_2=1$, oversampling ratios $ \alpha_1=0.1$, $\alpha_2 = 0.2$ and SNRs $\lambda_1 = \lambda_2 = 4$. When two tasks are highly correlated ($c \gtrsim 0.5$), the performance of task 1 is significantly improved.}
    \label{fig:2-tasks}
\end{figure}

 In Figure~\ref{fig:improve}, we compute the rate of error reduction in the target task as a result of transferring information from the source task. We found that MTL is most effective when the SNR of the target task is near the phase transition and is smaller than that of the source task, while the proportion of labeled data is low. 
\begin{figure}[!htb]
    \centering
    \includegraphics[width=0.45\textwidth]{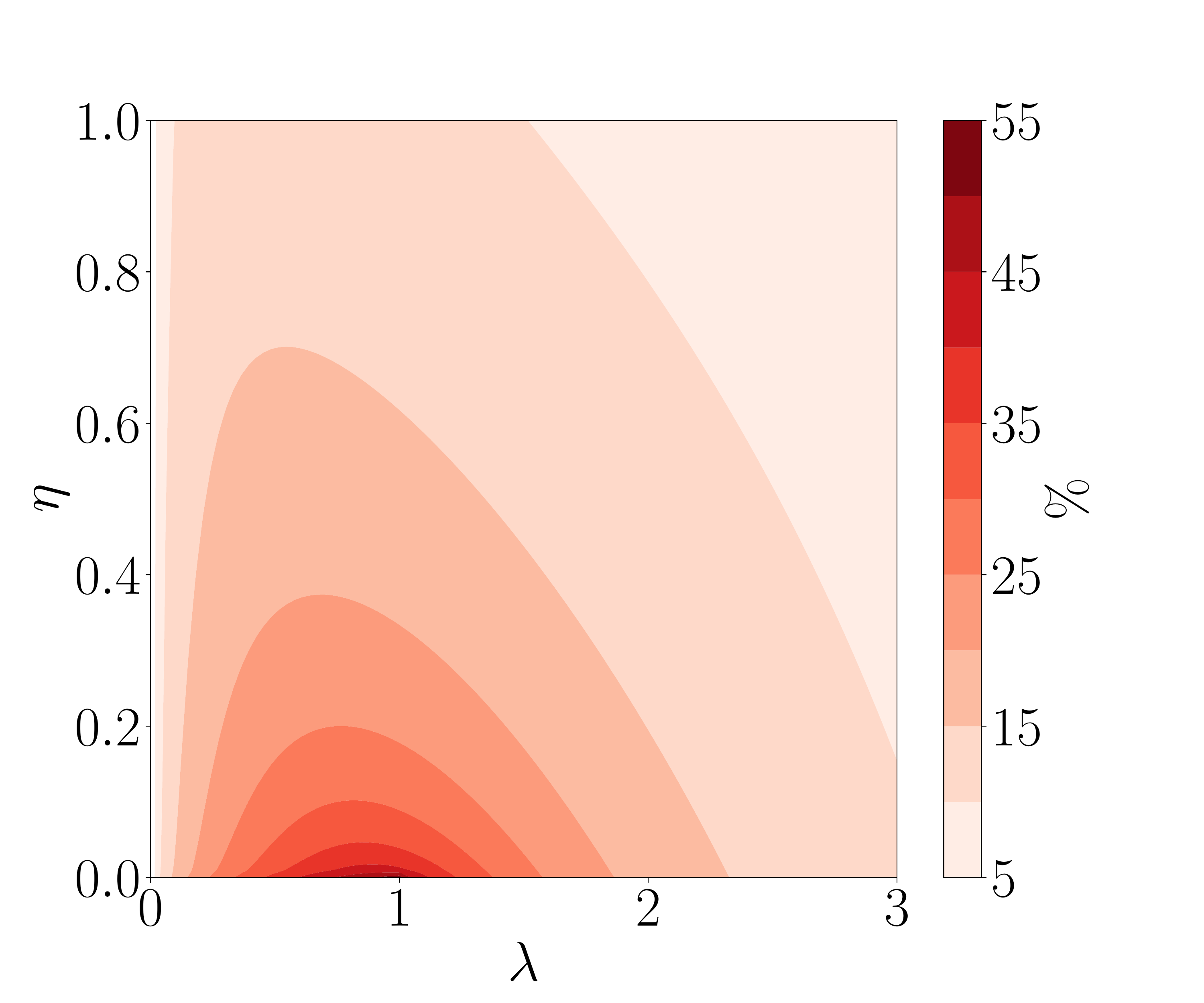}
    \caption{Percentage of reduction of Bayes risk in term of SNR and proportion of labeled data of the target task, with parameters  $c=0.8$, $N_1 = N_2 = D$, $ \lambda_1=2$, $ 0 \leq \lambda_2 \leq 3$, $ \eta_1=1$, $ 0 \leq \eta_2 \leq 1$. }
    \label{fig:improve}
\end{figure}

Intuitively, there are three reasons for this. Firstly, the labeled data from the target task is more valuable than that of source task, even in this case where two tasks are highly correlated ($c=0.8$). This leads to lower gain when the proportion of labeled data in the target task is high. Secondly, if the source task is more difficult than the target task, i.e. the SNR is higher in the target task, then the source task is not very useful. Finally, near the phase transition where the target task struggles, labeled data from the source task can offer valuable help.

\medskip


\section{CAVITY ARGUMENT}\label{cavity}
The various equations obtained in the paper are underpinned by the phenomenon of asymptotic equivalence that occurs in the high-dimensional limit. In this limit, a fairly complicated statistical model decouples into independent components, and the inference can be performed separately in each component. This decoupling phenomenon is proven using the so-called \emph{cavity method}. The following lemma plays a crucial role in the cavity argument presented in this paper:
\begin{lemma}\label{lem}
    Suppose we want to estimate  the signal $X \in \R$ with prior $P_X$ from the data $\vec Y$ that can be split into two parts as follows. The first part, denoted by $\vec Y^x$, consists of the following observation on $X$,
    \begin{align}
        \vec Y^x = X \vec U + \vec Z,
    \end{align}
    where
    \begin{itemize}
        \item $\vec U \in R^D$ is unknown with prior $P_U$,
        \item  $\vec Z \sim \mathcal N(0, I_{D})$,
        \item $X$, $\vec U$ and $\vec Z$ are independent.
    \end{itemize}
    The second dataset, denoted by $\vec Y^u$, is independent of $\vec X$. Suppose that the law $\vec U | \vec Y^u$ has the RS property with overlap $q$.  Then in the limit $D \ra \infty$,
    
    i) The posterior of $\vec X$ given $\vec Y$ is asymptotically equivalent to the law $\bar P$ defined as
    \begin{align}
    \frac{d \bar P(x | \vec Y)}{dP_X (x)}  \propto \exp \bround{ x \innero{\vec Y^x, \hat{\vec U}} - \frac{1}{2} q x^2  }
    \end{align}
    where $\hvec U = \E[\vec U|\vec Y]$. As a consequence, the statistics $S = \innero{\vec Y^x, \hvec U}$ is asymptotically sufficient for estimating $\vec X$ from $\vec Y$.

    ii) $S/\sqrt{q}$ converges in law to $\sqrt{q} X + \xi$, where $\xi$ follows standard normal distribution and is independent of $X$. As a result, estimating $X$ from $\vec Y$ is asymptotically equivalent to estimating $X$ from the output of a Gaussian channel with SNR $q$.
\end{lemma}

\begin{proof}
Since $ X$ is independent of $ \vec U$ and $ \vec Y^u$, we have
\begin{align*}
    \frac{d P(x | \vec Y)}{dP_X (x)} &= \int d P(\vec u|\vec Y^u) P(x| \vec u, \vec Y^x) \\
    &\propto \int d P(\vec u|\vec Y^u) \exp \bround{ x \innero{\vec Y^x, \vec u} - \frac{1}{2} x^2 \normo{\vec u}^2  } \\
    &:=\mathcal A
\end{align*}
Define 
\begin{align}
    \mathcal B = \exp \bround{ x \innero{\vec Y^x, \hat{\vec U}} - \frac{1}{2} q x^2  }
\end{align}
To prove (i), we will show that $\E[(\mathcal A - \mathcal B)^2] \ra 0$ in the high-dimensional limit $D \ra \infty$ for any value of $x$. To do this, it is sufficient to show that $\E[\mathcal A^2]$, $\E[\mathcal B^2]$ and $\E[\mathcal A \mathcal B]$ converge to the same limit, using the RS property of  $\vec U|\vec Y^u$. Indeed, $\E[\mathcal A^2]$ can be written as
\begin{align*}
    &\E \exp \bround{ \sum_{a=1}^2 x \innero{\vec Y^x, \vec U^a} - \frac{1}{2} x^2 \normo{\vec U^a}^2  }
\end{align*}
where $\vec U^1, \vec U^2$ are sampled independently from $\vec U|\vec Y^u$. Substituting $\vec Y^x = X \vec U + \vec Z$ into the previous expression, we obtain
\begin{align*}
    &\E \exp \bround{ \sum_{a=1}^2 xX \innero{\vec U, \vec U^a} + x \innero{\vec Z, \vec U^a} - \frac{1}{2} x^2 \normo{\vec U^a}^2 }
\end{align*}
Taking the expectation over $\vec Z$ and using the fact that $ \E[e^{\innero{\vec a, \vec Z}}] = e^{\frac{1}{2} \normo{\vec a}^2}$, we have
\begin{align*}
    \E[\mathcal A^2] = \E \exp \bround{ \sum_{a=1}^2 xX \innero{\vec U, \vec U^a} + x^2 \innero{\vec U^1, \vec U^2} }
\end{align*}
It follows from RS property of $\vec U|\vec Y^u$ that
\begin{align}\label{kq}
    \lim_{D \ra \infty} \E[\mathcal A^2] = \E \exp \round{ 2qXx + qx^2 }
\end{align}
To calculate the limits of $\mathbb{E}[\mathcal{AB}]$ and $\mathbb{E}[\mathcal{B}^2]$, we follow exactly the same procedure, which involves substituting the definition of $\vec{Y}^x$, taking the expectation over $\vec{Z}$, and using the RS property. This leads us to the same limit as (\ref{kq}), thereby proving (i).

It follows immediately from the asymptotic equivalence between $P(x|\vec Y)$ and $\bar P(x|\vec Y)$ that the statistics $\innero{\vec Y^x, \hvec U}$ is asymptotically sufficient for estimating $X$ from $\vec Y$. This means that all of the relevant information about $X$ can be extracted from $\innero{\vec Y^x, \hvec U}$ instead of from $\vec Y$, without any loss of information in high dimensional limit.

Now we have
\begin{align*}
    \innero{\vec Y^x, \hvec U} = \innero{X \vec U + \vec Z, \hvec U} = X \innero{\vec U, \hvec U} + \innero{\vec Z, \hvec U}.
\end{align*}
Given that $\innero{\vec Z, \hvec U} \sim \mathcal N(0, \normo{\hvec U}^2)$ and $\vec Z$ is independent of $X$, in the limit $D \rightarrow \infty$, this inner product converges in distribution to $\sqrt{q} \xi$, where $\xi$ is a standard normal random variable independent of $\vec X$. Therefore
\begin{align*}
    \frac{\innero{\vec Y^x, \hvec U}}{\sqrt q} \stackrel{d}{\longrightarrow} \sqrt{q} X + \xi, \quad D \ra \infty,
\end{align*}
which proves (ii) since the left hand side of the last expression is also a sufficient statistics of $X$ given $\vec Y$.
\end{proof}

To give an application of Lemma \ref{lem} and to familiarize readers with the cavity argument before delving into the proof of the main results in the paper, we will analyze the following tensor model studied in \cite{miolane2017fundamental}. Our goal is to estimate the signals $\vec{U}$ and $\vec{V}$ from the following observations:
\begin{align}
    Y_{ij} = \sqrt{\frac{\lambda}{N}} U_i V_j + Z_{ij}, i \in [N_u], j \in [N_v]
\end{align}
Here, we assume that $U_i \iid P_U, V_j \iid P_V$ and  the noises $Z_{ij}$ follow independent standard Gaussian distributions for all $i,j$. We study the model in the limit as $N, N_u, N_v$ tend to infinity with fixed ratios $ N_u/N \ra \alpha_u $ and $ N_v/N \ra \alpha_v$. Furthermore, we assume that $\vec U, \vec V$ and $\vec Z = (Z_{ij})$ are independent. It can be shown that both $N_u^{-1/2}\vec U|\vec Y$ and $N_v^{-1/2} \vec V|\vec Y$ satisfies the replica symmetry property, with overlaps $q_u$ and $q_v$ respectively. We will use Lemma \ref{lem} to derive the fixed point equations that satisfied by $q_u, q_v$.

Let $i \in [N_{u}]$ be fixed. The cavity method involves dividing the data $\vec{Y}$ into two parts. The first part, denoted as $\vec{Y}^1$, includes the observations related to $U_i$, given by
\begin{align} \label{fl}
    \vec Y_{i \cdot} = \sqrt{\frac{\lambda}{N}} U_i \vec V + \vec Z_{i \cdot}
\end{align}
while the remaining data is denoted as $\vec{Y}^2$. Since the dataset $\vec Y^1$ only contains an insignificant amount of information relevant to $\vec V$ (one can see that by comparing the sizes of $\vec Y^1$ and $\vec Y^2$), estimating $\vec V $ from $ \vec Y$ is essentially the same as estimating $\vec V$ from  $\vec Y^2$. Therefore, $N_v^{-1/2}\vec V|Y^2$ also satisfies the RS property with overlap $q_v$. It is easy to check that the Lemma \ref{lem} is applicable for this model, with $U_i $ and $\sqrt{\lambda/N} \vec V$ respectively playing the role of $X $ and $\vec U$ in the lemma. As a result, estimating $U_i$ from $\vec Y$ is asymptotically equivalent to estimating the signal $U_i$ from the output of a Gaussian channel with SNR $\lambda \alpha_v q_v$.

For distinct $i, k \in N_u$, since $\vec Z_{i\cdot}$ and $\vec Z_{k \cdot}$ are independent, it can be seen from the proof of Lemma \ref{lem}-ii that the noises $\xi_i$ and $\xi_{k}$ of the equivalent Gaussian channels associated with $U_i, U_k$ are independent. Therefore $\hvec U_i$, which depends on $\xi_i$ and $U_i$, are asymptotically independent for all $i$. By the law of large number
\begin{align}
    q_u = \lim_{N_u \ra \infty} \frac{1}{N_u} \sum_{i=1}^{N_u} \hat U_i^2 = F_U(\lambda  q_v)
\end{align}
where $F_U$ is the overlap function of the Gaussian channel with signal $U$.
Repeating the same argument for $V_j$ with $j \in N_v$, we obtain the fixed point equations for $q_u, q_v$:
\begin{align*}
    q_u = F_U(\lambda \alpha_v q_v) \\
    q_v = F_V(\lambda \alpha_u q_u)
\end{align*}
Note that fixed point equations may not uniquely determine overlaps, as they can have multiple solutions. However, rigorous methods (\cite{barbier2019adaptive}) demonstrate that overlaps can be uniquely determined as the minimax point of a certain function.

\section{PROOFS}

\subsection{Fixed point equations}
\textbf{Reformulation as a tensor model.} Let $ \tilde{\vec U}_t = \sqrt{D}\vec U_t$, it is shown in Appendix \ref{prior} that in the limit $ D \ra \infty$, $ \tilde U_{tj}$ are asymptotically Gaussian with covariance
\begin{align}\label{re}
    \E[\tilde U_{tj} \tilde U_{t'j'}] = C_{tt'}\delta_{jj'}
\end{align}

Let $ \vec W_t = \sqrt{D}\vec U_t/\sigma_t$, the original model can be written as a collection of one-dimensional Gaussian channels
\begin{align}\label{eq:channels}
Y_{ijt} = \frac{1}{ \sqrt{D}} V_{ti} W_{tj} + Z_{tij} 
\end{align}
for $ 1 \leq t \leq T, 1 \leq i \leq N_t, 1 \leq j \leq D$. As $ D \ra \infty$, the random variables $ W_{tj}$ are asymptotically Gaussian with covariance
\begin{align}
    \E[W_{tj}W_{t'j'}] = M_{tt'} \delta_{jj'}
\end{align}
where $M_{tt'} = C_{tt'}/(\sigma_t \sigma_{t'})$.

Next, the information conveyed by the labels can be absorbed into the prior distribution of $\vec V$. Specifically, if the value of $V_{ti}$ is unknown, then its prior remains uniform over $\br{-1,1}$. Otherwise, if it is known that $V_{ti}=1$, then the prior of $V_{ti}$ is $\delta(v-1)$. Note that in this case, the posterior coincides with the prior.

The RS property of $\sigma_t^{-2} \vec U_t|\vec Y$ implies that $D^{-1/2}\vec W_t|\vec Y$ also has the RS property with overlap $ \vec q_{ut}$.

In summary, the problem can be cast as a tensor model, whereby the objective is to estimate the signals $\vec V_t$ and $\vec W_t$ based on prior information regarding these vectors and noisy observations of the tensor products $\vec V_t \otimes \vec W_t$.

\textbf{Cavity argument.}  A crucial step in the analysis is to show that in the high-dimensional limit, estimating $\vec V_t$ and $\vec W_t$ given $\vec Y$ is asymptotically equivalent to estimating the coordinates of these vectors from Gaussian channels with independent noises. The original model is thus equivalent to a much more decoupled model and the inference can be done separately on each channel.

To obtain the fixed point equations, we follow the same approach as the example presented in Section \ref{cavity}. We assume that the proportion of unlabeled data is positive in any task. By taking the limit of these proportions to zero, we can derive the result for the supervised case. Fix $t \in [T]$ and $i \in [N_t]$ such that $V_{ti}$ is unknown. We divide the data $\vec Y$ into two parts: $\vec Y^1$ consisting of the observations concerning $V_{ti}$, namely
\begin{align}
    \vec Y_{ti} = \frac{1}{\sqrt{D}} V_{ti} \vec W_t + \vec Z_{ti} \nonumber
\end{align}
and the remaining data $\vec Y^2$. Since the dataset $\vec Y^1$ only contains an insignificant amount of information relevant to $\vec W_t$, estimating $\vec W_t $ from $ \vec Y$ is essentially the same as estimating $\vec W_t$ from  $\vec Y^2$. Therefore, $D^{-1/2}\vec W_t|\vec Y^2$ also satisfies the RS property with overlap $q_u$. It is easy to check that the Lemma \ref{lem} is applicable, with $V_{ti} $ and $D^{-1/2} \vec W_t$ respectively playing the role of $X $ and $\vec U$ in the lemma. As a result, estimating $V_{ti}$ from $\vec Y$ is asymptotically equivalent to estimating the signal $V_{ti}$ from the output of the Gaussian channel with SNR $ q_{ut}$. For distinct $i, k \in [N_t]$, since $\vec Z_{ti}$ and $\vec Z_{tk}$ are independent, it can be seen from the proof of Lemma \ref{lem}-ii that the noises $\xi_i$ and $\xi_{k}$ of the equivalent Gaussian channels associated with $V_{ti}, V_{tk}$ are also independent. Therefore $V_{ti}$, which depends on $\xi_i$ and $V_{ti}$, are asymptotically independent for all $i$ such that $V_{ti}$ is unlabeled. By the law of large number,
\begin{align}\label{n1}
    r_{vt} &:= \lim_{N_t \ra \infty} \frac{1}{(1-\eta_t) N_t} \sum_{i} \hat V_{ti}^2 \nn
    &= F_v(q_{ut})
\end{align}
where the sum is over all $i \in [N_t]$ such that $V_{ti}$ is unlabeled and $F_v$ is the overlap function of the Gaussian channel with Rademacher signal. From Appendix \ref{ra},
\begin{align}\label{n2}
    F_v(q) &= \E[\tanh(\sqrt{q}Z + q)], \quad Z \sim \mathcal N(0,1).
\end{align}

On the other hand, from the definition of $r_{vt}$, we have
\begin{align} \label{n3}
    q_{vt} = \eta_t + (1-\eta_t)r_{vt}
\end{align}
The fixed point equation (\ref{eq:fixed-point-2}) follows from (\ref{n1}), (\ref{n2}) and (\ref{n3}).

Following exactly the same cavity argument, the estimation of $W_{tj}$ given $ \vec Y$ is asymptotically equivalent the the estimation of the signal $W_{tj}$ from the output of the Gaussian channel with SNR $\alpha_t q_{vt}$. Moreover, the noises corresponding to the signals $W_{tj}$ and $W_{t'j'}$ are asymptotically independent for $(t,j) \neq (t', j')$. When $j \neq j'$, the signals $W_{tj}$ and $W_{t'j'}$ are independent. As a result, the inference on the equivalent Gaussian channels can be performed independently on groups of $T$ scalar Gaussian channels $(W_{tj})_{t=1}^T$. By the law of large number,
\begin{align}
    q_{ut} = \lim_{D \ra \infty} \frac{1}{D} \sum_{j=1}^D \hat W_{tj}^2 = F_{w,t}(\br{\alpha_t q_{vt}}_{t=1}^T)
\end{align}
where $F_{w,t}$ are overlap functions of the Gaussian channel with signal $ \mathcal N(0, \vec M)$. The explicit formula for $F_{w,t}$ are computed in Appendix \ref{gaussian-signals}, which gives the fixed point equation (\ref{eq:fixed-point-1}).

\subsection{Bayes risk and optimal algorithm}
Suppose we want to classify a new data point $ \vec Y_{\text{new}}$ in task $ t$
\begin{align}
    \vec Y_{\text{new}} &= V_{\text{new}} \vec U_t + \sigma_t \vec Z_{\text{new}}    
\end{align}
It is easy to check that Lemma \ref{lem} can be applied to this problem, with $V_{\text{new}}, \vec U_t$ playing the role of $X, \vec U$ in the lemma, as the posterior $\sigma_t^{-1}\vec U_t| \vec Y$ satisfies the RS property with overlap $q_{ut}$. As a result, in high dimensional limit, estimating $ V_{\text{new}}$ given $\vec Y, \vec Y_{\text{new}}$ is essentially the same as estimating the signal $V_{\text{new}}$ from the output of the Gaussian channel with SNR $q_{ut}$. This implies that the minimal classification error of $V_{\text{new}}$ is given by that of the Gaussian channel with Rademacher signal and SNR $q_{ut}$, which is (Appendix \ref{ra})
\begin{align*}
   1-\Phi \round{\sqrt{q_{ut}}},
\end{align*}
According to Lemma \ref{lem}, $S = \innero{\vec Y_{\text{new}}, \hvec U_t}/\sqrt{q_{ut}}$ is sufficient for estimating $V_{\text{new}}$. Moreover, $S$ converges in law to the output of the Gaussian channel with signal $V_{\text{new}}$ and SNR $q_{ut}$. The estimator that minimizes the Bayes risk for this channel is simply $ \sgn(S)$, which leads to the optimal estimator of $V_{\text{new}}$ as $\sgn(\innero{\vec Y_{\text{new}}, \hvec U_t})$. The next step is to determine the value of $\hvec U_t$. We will take advantage of the fact that the vectors $\vec U_t$ are asymptotically Gaussian, so our subsequent argument will rely on the reformulation (\ref{re}) of the model. We will need the following result
\begin{lemma}\label{ll}
The following collection of Gaussian channels 
    \begin{align}
        Y_i = c_i X_i + Z_i, \quad i=1 \ddd n
    \end{align}
    with inputs $X_i$, outputs $Y_i$, SNR $c_i^2$ and independent standard Gaussian noises $Z_i$, is equivalent to a single Gaussian channel with signal $X$, output $ \innero{\vec c, \vec Y}/\normo{\vec c}$ and SNR $ \sum_{i=1}^n c_i^2$. Moreover, 
\end{lemma}
\begin{proof}
It is straightforward to verify that the statistics  $S:= \innero{\vec c, \vec Y}/\normo{\vec c}$ is sufficient  for estimating $X$ from $\vec Y$. Moreover, $ S = \norm{\vec c}  X + \xi$ where $\xi = \norm{\vec c}^{-1} \innero{\vec c, \vec Z}$ is standard Gaussian and independent of $X$. This proves the claim of the lemma.
\end{proof}

\begin{remark} \label{rema}
    From the proof of Lemma \ref{ll} we can also see that the noise $\xi$ of the simplified channel comes from the noises of the original channels.
\end{remark}

The Lemma \ref{ll} implies that, for each $ (t, j)$ fixed, the following Gaussian channels
\begin{align*}
    Y_{tij} = \frac{1}{\sqrt{D}} V_{ti} W_{tj} + Z_{tij}, \quad i = 1 \ddd N_t
\end{align*}
which share the same signal $W_{tj}$, can be simplified into a single Gaussian channel with output $ \sqrt{N_t} \bar{Y}_{tj}$ and SNR $N_t/D\simeq \alpha_t$, where $\bar Y_{tj} $ is the $j$-th coordinate of the vector $\bar{\vec Y}_{t}$ in the algorithm.

For $(t,j)\neq(t',j')$, the noises of the simplified Gaussian channels associated with $W_{tj}$ and $ W_{t' j'}$ are independent, as a consequence of Remark \ref{rema}. Additionally, the signals $W_{tj}$ and $W_{t'j'}$ are independent if $j \neq j'$. Therefore, the inference on the simplified Gaussian channels can be carried out independently on each group of $ T$ channels with signals $(W_{tj})_{t=1}^T$. The MMSE estimator on each of these groups can be computed explicitly as
\begin{align*}
    (\hat{W}_{tj})_{t=1}^T = \vec B (\sqrt{N_t} \bar Y_{tj})_{t=1}^T
\end{align*}
where
\begin{align*}
   \vec B =  \vec M \vec D_{\vec \alpha}^{1/2} (\vec I + \vec D_{\vec \alpha}^{1/2} \vec M \vec D_{\vec \alpha}^{1/2})^{-1}
\end{align*}
(Appendix \ref{gaussian-signals}). Equivalently,
\begin{align*}
    \hat{\vec W}_{t} = \sum_s B_{ts} \sqrt{N_s} \bar{\vec Y}_s
\end{align*}
Dividing both sides by $ \sqrt{D}$ and using $ N_t/D \simeq \alpha_t$, we have
\begin{align}
    \tilde{\vec Y}_t: = \sigma_t^{-1}\hat{\vec U}_{t} \simeq \sum_s A_{ts} \bar{\vec Y}_s 
\end{align}
where $ A_{ts} = B_{ts} \sqrt{\alpha_s}$. Therefore,
\begin{align*}
    \vec A = \vec M\vec D_{\vec \alpha}(\vec I+\vec M\vec D_{\vec \alpha})^{-1}
\end{align*}
as given in the optimal algorithm. The optimal estimator for $V_{\text{new}}$ is $ \sgn(\innero{\vec Y_{\text{new}}, \hvec U_t } ) = \sgn(\innero{\vec Y_{\text{new}}, \tilde{\vec Y}_t })$.

\section{CONCLUSION}
This paper proposed a Gaussian mixture model of multitasking learning, in which each task is a semi-supervised classification problem. We derived an explicit formula for the Bayes risk, from which the behaviors of the model is studied through various numerical simulations. 

The model in this paper concerns with Gaussian and Rademacher random variables. However, our method also works for more general tensor models with random variables of finite second moments. 

\textbf{Acknowledgement.}
We would like to thank the reviewers for their valuable feedback and insightful comments, which have significantly contributed to the improvement of this paper. We would also like to express our appreciation to Malik Tiomoko for insightful discussions on the algorithmic aspects of the model and to Hugues Souchard de Lavoreille for his internship report, which the first author consulted numerous times while working on this paper. Our research is supported by MIAI.

\bibliography{main}

\begin{thebibliography}{23}
\providecommand{\natexlab}[1]{#1}
\providecommand{\url}[1]{\texttt{#1}}
\expandafter\ifx\csname urlstyle\endcsname\relax
  \providecommand{\doi}[1]{doi: #1}\else
  \providecommand{\doi}{doi: \begingroup \urlstyle{rm}\Url}\fi

\bibitem[Aubin et~al.(2020)Aubin, Krzakala, Lu, and
  Zdeborov{\'a}]{aubin2020generalization}
B.~Aubin, F.~Krzakala, Y.~Lu, and L.~Zdeborov{\'a}.
\newblock Generalization error in high-dimensional perceptrons: Approaching
  bayes error with convex optimization.
\newblock \emph{Advances in Neural Information Processing Systems},
  33:\penalty0 12199--12210, 2020.

\bibitem[Baik et~al.(2005)Baik, Arous, and P{\'e}ch{\'e}]{baik2005phase}
J.~Baik, G.~B. Arous, and S.~P{\'e}ch{\'e}.
\newblock Phase transition of the largest eigenvalue for nonnull complex sample
  covariance matrices.
\newblock \emph{The Annals of Probability}, 33\penalty0 (5):\penalty0
  1643--1697, 2005.

\bibitem[Barbier and Macris(2019)]{barbier2019adaptive}
J.~Barbier and N.~Macris.
\newblock The adaptive interpolation method: a simple scheme to prove replica
  formulas in bayesian inference.
\newblock \emph{Probability theory and related fields}, 174:\penalty0
  1133--1185, 2019.

\bibitem[Barbier et~al.(2017)Barbier, Macris, and Miolane]{barbier2017layered}
J.~Barbier, N.~Macris, and L.~Miolane.
\newblock The layered structure of tensor estimation and its mutual
  information.
\newblock In \emph{2017 55th Annual Allerton Conference on Communication,
  Control, and Computing (Allerton)}, pages 1056--1063. IEEE, 2017.

\bibitem[Barbier et~al.(2019)Barbier, Krzakala, Macris, Miolane, and
  Zdeborov{\'a}]{barbier2019optimal}
J.~Barbier, F.~Krzakala, N.~Macris, L.~Miolane, and L.~Zdeborov{\'a}.
\newblock Optimal errors and phase transitions in high-dimensional generalized
  linear models.
\newblock \emph{Proceedings of the National Academy of Sciences}, 116\penalty0
  (12):\penalty0 5451--5460, 2019.

\bibitem[Bishop(1998)]{bishop1998bayesian}
C.~Bishop.
\newblock Bayesian pca.
\newblock \emph{Advances in neural information processing systems}, 11, 1998.

\bibitem[Krzakala et~al.(2012)Krzakala, M{\'e}zard, Sausset, Sun, and
  Zdeborov{\'a}]{krzakala2012statistical}
F.~Krzakala, M.~M{\'e}zard, F.~Sausset, Y.~Sun, and L.~Zdeborov{\'a}.
\newblock Statistical-physics-based reconstruction in compressed sensing.
\newblock \emph{Physical Review X}, 2\penalty0 (2):\penalty0 021005, 2012.

\bibitem[Lelarge and Miolane(2019{\natexlab{a}})]{lelarge2019asymptotic}
M.~Lelarge and L.~Miolane.
\newblock Asymptotic bayes risk for gaussian mixture in a semi-supervised
  setting.
\newblock In \emph{2019 IEEE 8th International Workshop on Computational
  Advances in Multi-Sensor Adaptive Processing (CAMSAP)}, pages 639--643. IEEE,
  2019{\natexlab{a}}.

\bibitem[Lelarge and Miolane(2019{\natexlab{b}})]{lelarge2019fundamental}
M.~Lelarge and L.~Miolane.
\newblock Fundamental limits of symmetric low-rank matrix estimation.
\newblock \emph{Probability Theory and Related Fields}, 173\penalty0
  (3):\penalty0 859--929, 2019{\natexlab{b}}.

\bibitem[Lesieur et~al.(2016)Lesieur, De~Bacco, Banks, Krzakala, Moore, and
  Zdeborov{\'a}]{lesieur2016phase}
T.~Lesieur, C.~De~Bacco, J.~Banks, F.~Krzakala, C.~Moore, and L.~Zdeborov{\'a}.
\newblock Phase transitions and optimal algorithms in high-dimensional gaussian
  mixture clustering.
\newblock In \emph{2016 54th Annual Allerton Conference on Communication,
  Control, and Computing (Allerton)}, pages 601--608. IEEE, 2016.

\bibitem[Lesieur et~al.(2017)Lesieur, Miolane, Lelarge, Krzakala, and
  Zdeborov{\'a}]{lesieur2017statistical}
T.~Lesieur, L.~Miolane, M.~Lelarge, F.~Krzakala, and L.~Zdeborov{\'a}.
\newblock Statistical and computational phase transitions in spiked tensor
  estimation.
\newblock In \emph{2017 IEEE International Symposium on Information Theory
  (ISIT)}, pages 511--515. IEEE, 2017.

\bibitem[Loureiro et~al.(2021)Loureiro, Sicuro, Gerbelot, Pacco, Krzakala, and
  Zdeborov{\'a}]{loureiro2021learning}
B.~Loureiro, G.~Sicuro, C.~Gerbelot, A.~Pacco, F.~Krzakala, and
  L.~Zdeborov{\'a}.
\newblock Learning gaussian mixtures with generalized linear models: Precise
  asymptotics in high-dimensions.
\newblock \emph{Advances in Neural Information Processing Systems},
  34:\penalty0 10144--10157, 2021.

\bibitem[Mai and Couillet(2021)]{mai2021consistent}
X.~Mai and R.~Couillet.
\newblock Consistent semi-supervised graph regularization for high dimensional
  data.
\newblock \emph{J. Mach. Learn. Res.}, 22:\penalty0 94--1, 2021.

\bibitem[Mezard and Montanari(2009)]{mezard2009information}
M.~Mezard and A.~Montanari.
\newblock \emph{Information, physics, and computation}.
\newblock Oxford University Press, 2009.

\bibitem[Mignacco et~al.(2020)Mignacco, Krzakala, Lu, Urbani, and
  Zdeborova]{mignacco2020role}
F.~Mignacco, F.~Krzakala, Y.~Lu, P.~Urbani, and L.~Zdeborova.
\newblock The role of regularization in classification of high-dimensional
  noisy gaussian mixture.
\newblock In \emph{International Conference on Machine Learning}, pages
  6874--6883. PMLR, 2020.

\bibitem[Miolane(2017)]{miolane2017fundamental}
L.~Miolane.
\newblock Fundamental limits of low-rank matrix estimation: the non-symmetric
  case.
\newblock \emph{arXiv preprint arXiv:1702.00473}, 2017.

\bibitem[Paredes et~al.(2012)Paredes, Argyriou, Berthouze, and
  Pontil]{paredes2012exploiting}
B.~R. Paredes, A.~Argyriou, N.~Berthouze, and M.~Pontil.
\newblock Exploiting unrelated tasks in multi-task learning.
\newblock In \emph{Artificial intelligence and statistics}, pages 951--959.
  PMLR, 2012.

\bibitem[Polson and Scott(2011)]{polson2011data}
N.~G. Polson and S.~L. Scott.
\newblock Data augmentation for support vector machines.
\newblock \emph{Bayesian Analysis}, 6\penalty0 (1):\penalty0 1--23, 2011.

\bibitem[Ruder(2017)]{ruder2017overview}
S.~Ruder.
\newblock An overview of multi-task learning in deep neural networks.
\newblock \emph{arXiv preprint arXiv:1706.05098}, 2017.

\bibitem[Thrampoulidis et~al.(2020)Thrampoulidis, Oymak, and
  Soltanolkotabi]{thrampoulidis2020theoretical}
C.~Thrampoulidis, S.~Oymak, and M.~Soltanolkotabi.
\newblock Theoretical insights into multiclass classification: A
  high-dimensional asymptotic view.
\newblock \emph{Advances in Neural Information Processing Systems},
  33:\penalty0 8907--8920, 2020.

\bibitem[Tiomoko et~al.(2021{\natexlab{a}})Tiomoko, Couillet, and
  Pascal]{tiomoko2021pca}
M.~Tiomoko, R.~Couillet, and F.~Pascal.
\newblock Pca-based multi task learning: a random matrix approach.
\newblock \emph{arXiv preprint arXiv:2111.00924}, 2021{\natexlab{a}}.

\bibitem[Tiomoko et~al.(2021{\natexlab{b}})Tiomoko, Tiomoko, and
  Couillet]{tiomoko2021deciphering}
M.~Tiomoko, H.~Tiomoko, and R.~Couillet.
\newblock Deciphering and optimizing multi-task learning: a random matrix
  approach.
\newblock In \emph{ICLR 2021-9th International Conference on Learning
  Representations}, 2021{\natexlab{b}}.

\bibitem[Tipping and Bishop(1999)]{tipping1999probabilistic}
M.~E. Tipping and C.~M. Bishop.
\newblock Probabilistic principal component analysis.
\newblock \emph{Journal of the Royal Statistical Society: Series B (Statistical
  Methodology)}, 61\penalty0 (3):\penalty0 611--622, 1999.

\end{thebibliography}

\onecolumn
\appendix

\section{Setting and main result}
We summarize here the general setting and main results of the paper.
We consider $T$ tasks, where task $ t$ consists in classifying $ N_t$ data points in $ \R^D$ that belong to two different Gaussian clusters with the same covariance $ \sigma_t^2 I_D$.
The dataset of each task is partially labeled.
The model is studied in the high dimensional setting $ D \ra \infty$ with the following parameters supposed to be known:
\begin{itemize}
    \item $ \vec C = (C_{tt'})_{t,t'=1}^T$: task correlations, with $C_{tt}=1$ for all $t$.
    \item $ \alpha_t = \lim_{D \ra \infty}  N_t/D$: oversampling ratios
    \item $ \lambda_t = 1/\sigma_t^2$: signal-to-noise ratios (SNRs)
    \item $ \eta_t$: proportion of labeled data in task $ t$
\end{itemize}
We are interested in the minimal probability of misclassifying a new data point in task $ t$, i.e. the Bayes risk of task $ t$.

\begin{result} \label{theo-mains}
Under the setting of the model, as $ D \ra \infty$, the Bayes risk of task $ t$ converges to
\begin{align*}
   1-\Phi \round{\sqrt{q_{ut}}},
\end{align*}
where $\Phi(t)=\frac1{\sqrt{2\pi}}\int_{-\infty}^t e^{-x^2}dx$ and $ (q_{ut}, q_{vt})_{t=1}^T$ is the stable solution of the system of equations
\begin{subequations}
    \begin{align}
    q_{ut} &= [\vec M-\vec M(\vec I+\vec D\vec M)^{-1}]_{tt} \label{eq:fixed-point-1s}\\
    q_{vt} &= \eta_t + (1-\eta_t)F(q_{ut}) \label{eq:fixed-point-2s}
    \end{align}
\end{subequations}
with
\begin{align*}
    \vec M &= \br{ C_{tt'}/\sigma_t \sigma_{t'} }_{t,t'=1}^T\\
    \vec D &= \diag \bro{ \alpha_t q_{vt}}_{t=1}^T \\
    F(q) &= \E[\tanh(\sqrt{q}Z + q)], \quad Z \sim \mathcal N(0,1).
\end{align*}
\end{result}

\section{Special cases} \label{special}
We check the main result with the following special cases.

\subsection{Uncorrelated tasks}
We consider here the case in which $ C_{tt'}=0$ for all $ t \neq t'$, the matrix $\vec M$ is diagonal and we obtain the following equations for each $ t$
\begin{align*}
    q_{ut} &= \frac{1}{\sigma_1^2} \frac{\alpha_t q_{vt}}{\sigma_t^2 + \alpha q_{vt}} \\
    q_{vt} &= \eta_t + (1-\eta_t)F(q_{ut})
\end{align*}
which is the same as the fixed point equations when the tasks are learned separately.

\subsection{The data for each task follows the same distribution}

We consider here the case in which $C_{tt'}=1$ and $\sigma_t = \sigma$ for all $ t,t' = 1 \ddd T$. We have
\begin{align*}
    \vec M = \frac{1}{\sigma^2} \1 \1^T, \quad \vec D \vec M = \frac{\vec u \1^T}{\sigma^2}
\end{align*}
where $ \vec u = (\alpha_t q_{vt})_{t=1}^\top$ and $ \1 = \underbrace{(1 \ddd 1)^\top}_{T \text{ } 1s}$. Applying the formula
\begin{align*}
    (\vec I + \vec u \vec v^T)^{-1} = \vec I - \frac{\vec u \vec v^T}{1+\vec u^T \vec v}
\end{align*}
for $ \vec v = \1/\sigma^2$, we obtain
\begin{align*}
    (\vec I+\vec D \vec M)^{-1} = \vec I - \frac{\vec u\1^T}{\sigma^2+\vec u^T \1}
\end{align*}
so
\begin{align}
    \vec M - \vec M(\vec I + \vec D \vec M)^{-1} = \frac{1}{\sigma^2} \frac{\vec u^T \1}{\sigma^2+\vec u^T \1} \1\1^T
\end{align}
It follows from the equation (\ref{eq:fixed-point-1s}) that for all $t$,
\begin{align}\label{jn1}
    q_{ut} =  \frac{1}{\sigma^2} \frac{\vec u^T \1}{\sigma^2+\vec u^T \1} := q_u
\end{align}
Define $ \alpha, \eta$  as
\begin{align}\label{003}
    \alpha = \sum_t \alpha_t, \quad \alpha \eta = \sum_t \alpha_t \eta_t
\end{align}
We have
\begin{align}\label{jn2}
    \vec u^T \1 &= \sum_{t} \alpha_t q_{ut} \nn
    &= \sum_t \alpha_t (\eta_t +(1-\eta_t) F(q_u) ) \nn
    &= \alpha \eta + \alpha(1-\eta) F(q_u) \nn
    &= \alpha q_v
\end{align}
where $q_v$ is defined as
\begin{align}\label{001}
    q_v = \eta + (1-\eta) F(q_u)
\end{align}
then from (\ref{jn1}) and (\ref{jn2}), we have
\begin{align}\label{002}
    q_u &= \frac{1}{\sigma^2}\frac{\alpha q_v}{\sigma^2+\alpha q_v}
\end{align}
Since (\ref{001}) and (\ref{002}) are exactly the equations for the case of single task learning with parameters $ \alpha$ and $ \eta$, the multitask learning problem is reduced to one single task with parameters $ \alpha, \eta$ given by (\ref{003}).

\section{Unsupervised learning and phase transition}
\subsection{Region of impossible recovery}

In the unsupervised case, the fixed point equations are
\begin{subequations}\label{005}
    \begin{align}
    q_{ut} &= [\vec M-\vec M(\vec I+\vec D\vec M)^{-1}]_{tt} \label{jd}\\
    q_{vt} &= F(q_{ut})
    \end{align}
\end{subequations}
which always admits $ (\vec q_u, \vec q_v)=(\vec 0, \vec 0)$ as solution.  The classification is impossible if and only if this solution is stable. To analyze the stability of (\ref{005}) around zero, let $ q_{ut}, q_{vt} = O(h)$ where $ h \ra 0$. For vectors $A$ and $B$ of the same dimension, we denote $ A \simeq B$ if $ |A-B| \simeq O(h^2)$, where $ |\,.\,|$ denotes the Euclidean norm. From 
\begin{align}
    F(q) = \E[\tanh(\sqrt{q}Z + q)],
\end{align}
(Appendix \ref{ra}), using the Taylor expansion $ \tanh(x) = x - x^3/3 + o(x^3)$, we get
\begin{align*}
    q_{vt} = F(q_{ut}) \simeq q_{ut} 
\end{align*}
On the other hand,
\begin{align*}
    q_{ut} &= [\vec M-\vec M(\vec I+\vec D\vec M)^{-1}]_{tt} \\
    &\simeq [\vec M-\vec M(\vec I-\vec D\vec M)]_{tt} \\
    &= [\vec M \vec D \vec M]_{tt} \\
    & = \sum_{s=1}^T M_{ts}^2 \alpha_{s} q_{vs}
\end{align*}
Let
\begin{align}
    \vec P = (M_{ts}^2 \alpha_{s})_{s,t=1}^T  = \bround{ \frac{C_{ts}^2}{\sigma_t^2 \sigma_{s}^2} \alpha_{s} }_{s,t=1}^T = ( \lambda_s \lambda_t C_{st}^2 \alpha_s )_{s,t=1}^T
\end{align}
In a small neighborhood of $ (\vec 0, \vec 0)$, the system of equations can be approximated up to an error of $O(h^2)$ by
\begin{align}
    \vec q_{v} &= \vec q_{u}\\
    \vec q_{u} &= \vec P \vec q_v
\end{align}
Therefore the fixed point $(\vec 0, \vec 0)$ is stable if and only if the module of each eigenvalue of $ \vec P$ is not larger than $ 1$. Using the property that $AB$ and $BA$ has the same eigenvalues for general square matrices $ A, B$, the matrix $ \vec P$ has the same eigenvalues as the following symmetric matrix
\begin{align}
    \vec R = ( \sqrt{\alpha_s \alpha_t} \lambda_s \lambda_t C_{st}^2 )_{s,t=1}^T
\end{align}
Note that $ \vec R$ is a positive semidefinite (p.s.d) matrix, since it can be written as Hadamard product of p.s.d. matrices. Therefore, the classification is impossible if and only if all eigenvalues of $ \vec R$ are not greater than $ 1$. 

When $ C_{tt'}=c$ for all $ t \neq t'$ and $ \lambda_t=\lambda, \alpha_t=1$ for all $ t$, we have
\begin{align}
    \vec R = \lambda^2(c^2 \1\1^T + (1-c^2)I ) 
\end{align}
Note that the matrix $ \1\1^T$ has eigenvalues $ 0\ddd 0, T$, so the largest eigenvalue of $ \vec R$ is $ \lambda^2(1+(T-1)c^2)$, from with we obtain the condition for impossible classification
\begin{align}
    \lambda^2(1+(T-1)c^2) \leq 1
\end{align}
which becomes $ \lambda \leq 1$ for the special case $ T=1$.

When $ T=2$ with task correlation $ c$ and $ \alpha_1=\alpha_2=1$, we have
\begin{align}
    \vec R = \pmat{ \lambda_1^2 & c^2 \lambda_1 \lambda_2 \\ c^2 \lambda_1 \lambda_2 & \lambda_2^2 }
\end{align}
It is clear that the $ (\lambda_1, \lambda_2)$-domain of impossible classification is a subset of $ [0,1]^2$, otherwise at least one task is achievable. All eigenvalues of $ \vec R$ are less than 1 if and only if $ \tr(\vec I-\vec R) \geq 0$ and $ \det(\vec I - \vec R) \geq 0$. The first condition is already satisfied for $ (\lambda_1, \lambda_2) \in [0,1]^2$ while the second condition is equivalent to
\begin{align}
    (1-\lambda_1^2)(1-\lambda_2^2) \leq c^4 \lambda_1^2 \lambda_2^2
\end{align}

\subsection{Connected tasks are either all feasible or impossible}
In the unsupervised case, tasks are considered connected if any two tasks are directly or indirectly correlated through other tasks. We will prove that if tasks are connected, then either all tasks are feasible or all tasks are impossible. As a reminder, for any task $t$, the value of $q_{ut}$ is always non-negative. If $q_{ut} = 0$, then the task $t$ is impossible; otherwise, it is feasible.

Consider $T$ Gaussian channels with outputs $(Y_t)_{t=1}^T$, signals $(X_t)_{t=1}^T$ having joint distribution $\mathcal{N}(0, \vec M)$ and independent standard Gaussian noises. The SNRs for each channel are $(\alpha_t q_{vt})_{t=1}^T$. Then the right-hand side of (\ref{jd}) corresponds to the overlap between the signal $X_t$ and its MMSE estimator (Appendix \ref{gaussian-signals}).

Suppose by contradiction that the tasks can be split into non-empty sets such $S$ and $S'$ such that $q_{ut}=0$ for all $t \in S$ while $q_{ut} > 0$ for all $t \in S'$. Since the tasks are connected, there exists correlated tasks $t,t'$ such that $t \in S, t' \in S'$. Therefore, there exists $t,t'$ such that $q_{ut}=0, q_{ut'}>0 $ and $C_{tt'} \neq 0$. 

Since $\E[X_t X_{t'}] = M_{tt'} = C_{tt'}/(\sigma_t \sigma_{t'}) \neq 0$, $X_t$ is correlated with $X_{t'}$. Moreover, as $ q_{vt'} = F(q_{ut'})$ and $ q_{ut'} > 0$, we have $ q_{vt'} > 0$. Therefore $X_t$ is not independent of $\vec Y = \br{\sqrt{\alpha_s q_{vs}} X_s + Z_s}_{s=1}^T$, leading to $q_{ut} = \E[X_t \E[X_t|\vec Y]] > 0$, a contradiction.

\section{Estimating model parameters from data} \label{C}
Although it is assumed that the model parameters $ \vec C$ and $ (\sigma_t)$ are available for the analysis, we show here that they can indeed be estimated with vanishing errors as $ D \ra \infty$, given that a positive fraction of labeled data is available in each task, i.e. $ \eta_t > 0$ for all $ t$. First consider the supervised learning case. Let
\begin{align}
    \bar{\vec Y}_t = \frac{1}{N_t} \sum_{i=1}^{N_t} V_{ti} \vec Y_{ti}
\end{align}
Then we have
\begin{align}\label{eq:ybar}
    \bar{\vec Y}_t = \vec U_t + \sqrt{\frac{\sigma_t^2}{N_t}} \bar{\vec Z}_t
\end{align}
where
\begin{align}
    \bar{\vec Z}_t = \frac{1}{\sqrt{N_t}} \sum_{i=1}^{N_t} V_{ti} \vec Z_{ti}
\end{align}
It is clear that $ \bar{\vec Z}_t \iid \mathcal N(0, I_D)$ for $ t = 1 \ddd T$.
\noindent We consider the following estimator of $ C_{tt'}$ for $ t \neq t'$:
\begin{align}
    \hat C_{tt'} = \innero{\bar{\vec Y}_t, \bar{\vec Y}_{t'}}
\end{align}
Insert (\ref{eq:ybar}) into the definition of $ \hat C_{tt'}$ and use the fact that $ \innero{ \bar{\vec Z}_t, \bar{\vec Z}_{t'} }=O(\sqrt{D})$, $ \innero{ \bar{\vec U}_t, \bar{\vec Z}_{t'} }=O(1)$, which are direct consequences of Central Limit Theorem, we obtain $ \hat C_{tt'} = C_{tt'} + O(D^{-1/2}) $. Moreover
\begin{align}
    \normo{\bar{\vec Y}_t}^2 = 1 + \frac{\sigma_t^2}{\alpha_t} + O(D^{-1/2}),
\end{align}
from which $ \sigma_t$ can also be estimated.

In the case where the proportion of labeled data is positive for all tasks, we can restrict the above estimators on the labeled data and obtain the approximate values of $ \vec C$ and $ (\sigma_t)$ with errors converging to zero when $ D \ra \infty$.

\section{Simple Gaussian channels}\label{sim}
\subsection{Rademacher signal.}\label{ra}
Consider the Gaussian channel given by
\begin{align}
Y = \sqrt{\lambda} X + Z,
\end{align}
where the Rademacher signal $X$ takes values of $1$ and $-1$ with equal probabilities and the standard Gaussian noise $Z$ is independent of $X$. We have
\begin{align}
    P(x|Y) &= \frac{P(x)P(Y|x)}{P(Y)} \nn
    &\propto e^{-(Y-\sqrt{\lambda}x)^2/2} \nn
    &\propto e^{\sqrt{\lambda}Yx},
\end{align}
from which we obtain the posterior distribution as
\begin{align}
    P(x|Y) = \frac{e^{\sqrt{\lambda}Yx}}{2 \cosh(\sqrt{\lambda} Y)}
\end{align}
and the MMSE estimator $\hat X_{\text{MMSE}} = \E[X|Y]$ as
\begin{align}
    \hat X = \sum_{x=\pm 1} x P(x|Y) = \tanh (\sqrt \lambda Y).
\end{align}
The overlap between the MMSE estimator and the signal is therefore
\begin{align}
    \E[X \hat X_{\text{MMSE}}] &=  \E[X \tanh(\sqrt \lambda (\sqrt{\lambda}X+Z)]  \nn
    &= \frac{1}{2}\E[\tanh( \lambda + \sqrt{\lambda}Z )] - \frac{1}{2}\E[\tanh( -\lambda + \sqrt{\lambda}Z )] \nn
    &= \frac{1}{2}\E[\tanh( \lambda + \sqrt{\lambda}Z )] - \frac{1}{2}\E[\tanh( -\lambda - \sqrt{\lambda}Z )] \nn
    &= \E[\tanh(\sqrt{\lambda} Z + \lambda)]
\end{align}
Next, the error $\mathbb P(\hat X \neq X)$ for any estimator $\hat X$ of $X$ is minimized by the maximum-likelihood estimator:

\begin{align}
\hat X_{\text{ML}} &= \argmax_{x = \pm 1} P(x,Y) \nn
&= \argmax_{x = \pm 1} e^{\sqrt \lambda Y x}
\end{align}

This gives us the maximum-likelihood estimator as:

\begin{align}
\hat X_{\text{ML}} = \text{sgn}(Y).
\end{align}
The Bayes risk is therefore
\begin{align*}
    \mathbb P(X \neq \hat X_{ML}) &= \frac{1}{2} \mathbb P(X=1, \hat X_{ML}=-1) + \frac{1}{2} \mathbb P(X=-1, \hat X_{ML}=1) \\
    &= \frac{1}{2} \mathbb P(X=1, Y<0) + \frac{1}{2} \mathbb P(X=-1, Y>0) \\
    &= \mathbb P(X=-1, Y>0) \\
    &= \mathbb P(Z > \sqrt{\lambda})
\end{align*}

\subsection{Correlated Gaussian signals} \label{gaussian-signals}

Consider $T$ Gaussian channels, where the signals $X_1,\dots,X_T$ have a joint distribution of $\mathcal{N}(0,\vec M)$ and are independent of Gaussian noises $Z_1,\dots,Z_T$ that are independently distributed as $\mathcal{N}(0,1)$. Specifically, we have:
\begin{align*}
Y_t = \sqrt{\lambda_t} X_t + Z_t, \quad t = 1,\dots,T.
\end{align*}

Let $ \hat{X}_t = \EE{ X | \vec Y}$ be the MMSE estimator for $ X_t$. Since $(X_t, Y_1 \ddd Y_T)$ is a Gaussian vector, $ \hat X_t$ is a linear combination of $ Y_1 \ddd Y_T$. Therefore
\begin{align*}
    \text{MMSE}_t &:= \E[(X_t-\hat X_t)^2]  \\
    &= \min_{\vec \beta_t \in \R^T} \EE{(X_t - \innero{\vec \beta_t, \vec Y} )^2}.
\end{align*}
This can be written as a quadratic optimization problem
\begin{align*}
    \text{MMSE}_t = \min_{\vec \beta_t \in \R^T} \br{ M_{tt} - 2 \vec a_t^T \vec \beta_t + \vec \beta_t^T \vec A \vec \beta_t  }
\end{align*}
with
\begin{align*}
    \vec a_t &= \round{\EE{X_t Y_s}}_{s=1}^T = \round{\sqrt{\lambda_t} M_{ts}}_{s=1}^T = \vec D_{\vec \lambda}^{1/2} \vec M \vec e_t \\
    \vec A &= \round{\EE{ Y_s Y_{s'} }}_{s,s'=1}^T = \round{ \sqrt{\lambda_s \lambda_{s'}}  M_{ss'} + \delta_{ss'} }_{s,s'=1}^T = \vec I + \vec D_{\vec \lambda}^{1/2} \vec M \vec D_{\vec \lambda}^{1/2}. \\
\end{align*}
This optimization problem admits a unique minimizer $ \vec \beta_t = \vec A^{-1} \vec a_t$, from which we obtain 
\begin{align}
    \hat{\vec X} &= \vec M \vec D_{\vec \lambda}^{1/2} (\vec I + \vec D_{\vec \lambda}^{1/2} \vec M \vec D_{\vec \lambda}^{1/2})^{-1} \vec Y\\
    \text{MMSE}_t &= [\vec M (\vec I + \vec D_{\vec \lambda} \vec M)^{-1}]_{tt} \\
    \E[X_t \hat X_t] &= [\vec M - \vec M (\vec I + \vec D_{\vec \lambda} \vec M)^{-1}]_{tt}. \label{eq:overlap-2}
\end{align}

\section{The uniform prior is asymptotically Gaussian}\label{prior}
To generate $(\vec U_1, \dots, \vec U_T)$ according to the prior distribution specified in the model, we follow these steps:
\begin{enumerate}
    \item Generate $\vec Z_1 \ddd \vec Z_T \iid \mathcal N(\vec 0, I_D)$.
    \item Orthonormalize $\vec Z_1 \ddd \vec Z_T$ using Gram-Schmidt process, we obtain orthonormal vectors $\vec S_1 \ddd \vec S_T$
    \item $(\vec U_1 \ddd \vec U_T) = (\vec S_1 \ddd \vec S_T) \vec C^{1/2}$, where $(\vec U_1 \ddd \vec U_T)$ denotes the $D \times T$ matrix with columns $\vec U_1 \ddd \vec U_T$.
\end{enumerate}
In the high dimensional limit, the vector $\vec Z_1 \ddd \vec Z_T$ are asymptotically orthogonal, so the orthonormalizing step produces approximately $ n^{-1/2}(\vec Z_1 \ddd \vec Z_T)$, which implies that if $\vec W_t = \sqrt D \vec U_t$, then $\vec W_t$'s are asymptotically Gaussian with covariance
\begin{align}
    \E[W_{ti} W_{t'j}] = \delta_{ij} C_{tt'}
\end{align}
It is worth noting that this is a direct consequence of the equivalence between the canonical and microcanonical ensembles in statistical physics.

\end{document}